\documentclass[sigconf]{acmart}

\usepackage{booktabs} 

\usepackage{algcompatible}
\usepackage{float}
\usepackage{amssymb,amsmath,amsthm,amsfonts}
\newtheorem{theorem}{Theorem}

\newtheorem{definition}{Definition}
\usepackage{graphicx}
\usepackage{amsmath}
\usepackage{subcaption}
\usepackage{algorithm}
\usepackage{algcompatible}
\usepackage{xcolor}
\usepackage[bbgreekl]{mathbbol}
\usepackage{balance}

\newcommand{\R}{\mathbb{R}}

\newcommand{\1}{\boldsymbol{1}}
\newcommand{\bb}{\boldsymbol{b}}
\newcommand{\bx}{\boldsymbol{x}}
\newcommand{\bv}{\boldsymbol{v}}
\newcommand{\bu}{\boldsymbol{u}}
\newcommand{\bz}{\boldsymbol{z}}
\newcommand{\omegab}{\boldsymbol{\omega}}
\newcommand{\phib}{\boldsymbol{\phi}}

\newcommand{\bX}{\mathbf{X}}
\newcommand{\bW}{\mathbf{W}}
\newcommand{\bD}{\mathbf{D}}
\newcommand{\bL}{\mathbf{L}}
\newcommand{\bZ}{\mathbf{Z}}
\newcommand{\bU}{\mathbf{U}}
\newcommand{\bV}{\mathbf{V}}
\newcommand{\bI}{\mathbf{I}}
\newcommand{\bM}{\mathbf{M}}
\newcommand{\cM}{\mathcal{M}}
\newcommand{\bUh}{\widehat{\bU}}

\newcommand{\cMh}{\widehat{\cM}}

\newcommand{\bPhi}{\mathbf{\Phi}}
\newcommand{\tr}{\textnormal{trace}}
\newcommand{\cC}{\mathcal{C}}

\newcommand{\bdelta}{\boldsymbol{\delta}}
\newcommand{\bphi}{\boldsymbol{\phi}}

\newcommand{\0}{\boldsymbol{0}}

\begin{document}
\title{Scalable Spectral Clustering Using Random Binning Features}


\author{Lingfei Wu}
\affiliation{\institution{IBM Research AI}}
\email{wuli@us.ibm.com}

\author{Pin-Yu Chen}
\affiliation{\institution{IBM Research AI}}
\email{pin-yu.chen@ibm.com}

\author{ Ian En-Hsu Yen}
\affiliation{\institution{Carnegie Mellon University}}
\email{eyan@cs.cmu.edu}

\author{Fangli Xu }
\affiliation{\institution{College of William and Mary}}
\email{fxu02@email.wm.edu}

\author{Yinglong Xia}
\affiliation{\institution{Huawei Research}}
\email{yinglong.xia.2010@ieee.org}

\author{Charu Aggarwal}
\affiliation{\institution{IBM Research AI}}
\email{charu@us.ibm.com}

\renewcommand{\shortauthors}{Lingfei Wu, Pin-Yu Chen, Ian En-Hsu Yen, et al.}

\begin{abstract}
Spectral clustering is one of the most effective clustering approaches that capture hidden cluster structures in the data. However, it does not scale well to large-scale problems due to its quadratic complexity in constructing similarity graphs and computing subsequent eigendecomposition. Although a number of methods have been proposed to accelerate spectral clustering, most of them compromise considerable information loss in the original data for reducing computational bottlenecks. In this paper, we present a novel scalable spectral clustering method using \emph{Random Binning features} (RB) to simultaneously accelerate both similarity graph construction and the eigendecomposition. Specifically, we  implicitly approximate the graph similarity (kernel) matrix by the inner product of a large sparse feature matrix generated by RB. Then we introduce a state-of-the-art SVD solver to effectively compute eigenvectors of this large matrix for spectral clustering. Using these two building blocks, we reduce the computational cost from quadratic to linear in the number of data points while achieving similar accuracy. Our theoretical analysis shows that spectral clustering via RB converges faster to the exact spectral clustering than the standard Random Feature approximation. Extensive experiments on 8 benchmarks show that the proposed method either  outperforms or matches the state-of-the-art methods in both accuracy and runtime. Moreover, our method exhibits linear scalability in both the number of data samples and the number of RB features.
\end{abstract}

\copyrightyear{2018} 
\acmYear{2018} 
\setcopyright{acmlicensed}
\acmConference[KDD '18]{The 24th ACM SIGKDD International Conference on Knowledge Discovery \& Data Mining}{August 19--23, 2018}{London, United Kingdom}
\acmBooktitle{KDD '18: The 24th ACM SIGKDD International Conference on Knowledge Discovery \& Data Mining, August 19--23, 2018, London, United Kingdom}
\acmPrice{15.00}
\acmDOI{10.1145/3219819.3220090}
\acmISBN{978-1-4503-5552-0/18/08}

%
%



\maketitle

\section{Introduction}
Clustering is one of the most fundamental problems in machine learning and data mining tasks. In the past two decades, spectral clustering (SC) \cite{shi2000normalized,ng2002spectral,von2007tutorial,chen2017revisiting} has shown great promise for learning hidden cluster structures from data. The superior performance of SC roots in exploiting non-linear pairwise similarity between data instances, while traditional methods like K-means heavily rely on Euclidean geometry and thus place limitations on the shape of the clusters \cite{fowlkes2004spectral,yan2009fast,chen2016phase}. However, SC methods are typically not the first choice for large-scale clustering problems since modern datasets exhibit great challenges in both computation and memory consumption for computing the pairwise similarity matrix $\bW \in \R^{N \times N}$, where $N$ denotes the number of data points. In particular, given a data matrix $\bX \in \R^{N \times d}$ whose underlying data distribution can be represented as  $K$ weakly inter-connected clusters, it requires $O(N^2)$ space to store the matrix and $O(N^2d)$ complexity to compute $\bW$, and at least takes $O(KN^2)$ or $O(N^3)$ complexity to compute $K$ eigenvectors of the corresponding graph Laplacian matrix $\bL$, depending on whether an iterative or a direct eigensolver is used. To accelerate SC, many efforts have been devoted to address the following computational bottlenecks: 1) pairwise similarity graph construction of $\bW$ from the raw data $\bX$, and 2) eigendecomposition of  the graph Laplacian matrix $\bL$. 

A number of methods have been proposed to accelerate the eigendecomposition, e.g.,  randomized sketching and power method \cite{gittens2013approximate,lin2010power},  sequential reduction algorithm toward an early-stop strategy \cite{chen2006fast,liu2007fast}, and graph filtering of random signals \cite{tremblay2016compressive}. However,  these approaches only partially reduce the computation cost of the eigendecomposition, since the construction of similarity graph matrix $\bW$ still requires quadratic complexity for both computation and memory consumption. 

Another family of research is the use of Landmarks or representatives to jointly improve the computation efficiency of the similarity matrix $\bW$ and the eigendecomposition of $\bL$. One strategy is performing random sampling or K-means on the dataset to select a small number of representative data points and then employing SC on the reduced dataset \cite{yan2009fast,shinnou2008spectral}. Another strategy \cite{sakai2009fast,chen2011large,liu2013large,li2016scalable} is approximating the similarity matrix $\bW$ by a low rank affinity matrix $\bZ \in \R^{N \times R}$, which is computed via either random projection or a bipartite graph between all data points and selected anchor points \cite{liu2010large}. Furthermore, a heuristic that only selects a few nearest anchor points has been applied to build a KNN-based sparse graph similarity matrix. Despite promising results in accelerating SC, these approaches disregard considerable information in the raw data and may lead to degrading clustering performance. More importantly, there is no guarantee that these heuristic methods can approach the results of standard SC.

Another line of research \cite{fowlkes2004spectral,chitta2012efficient,chitta2011approximate,wu2018d2ke,wu2018random} focuses on leveraging various kernel approximation techniques such as Nystr{\"o}m  \cite{williams2001using}, Random Fourier \cite{rahimi2008random,wu2016revisiting,chen2016efficient}, and random sampling to accelerate similarity matrix construction and the eigendecomposition at the same time. The pairwise similarity (kernel) matrix $\bW$ (a weighted fully-connected graph) is then directly approximated by an inner product of the feature matrix $\bZ$ computed from the raw data. Although a KNN-based graph construction allows efficient sparse representation, pairwise method takes full advantage of more complete information in the data and takes into account the long-range connections \cite{fowlkes2004spectral,chen2011parallel}. The drawback of pairwise methods is the high computational costs in requiring the similarity between every pair of data points. Fortunately, we present an approximation technique applicable to SC that significantly alleviates this computational bottleneck. As our work focuses on enabling the scalability of SC using RB, the case of robust spectral clustering on noisy data, such as \cite{bojchevski2017robust}, 
could be applied but is beyond the scope of this paper.

In this paper, inspired by recent advances in the fields of kernel approximation and numerical linear algebra \cite{rahimi2008random,wu2016revisiting,wu2015preconditioned,wu2017primme_svds,chen2018incremental}, we present a scalable spectral clustering method and theoretic analysis to circumvent the two computational bottlenecks of SC in large-scale datasets. We highlight the following main contributions:

\begin{enumerate}
\item  We present for the first time the use of \emph{Random Binning features} (RB) \cite{rahimi2008random,wu2016revisiting} for scaling up the graph construction of similarity matrix in SC, which is implicitly approximated by the inner product of the RB sparse feature matrix $\bZ \in \R^{N \times D}$ derived from the raw dataset, where each row has $nnz(\bZ(i,:)) = R$. To this end, we reduce the computational complexity of the pairwise graph construction from $O(N^2d)$ to $O(NRd)$ and memory consumption from $O(N^2)$ to $NR$. 

\item We further show how to make full use of state-of-the-art near-optimal eigensolver PRIMME \cite{stathopoulos2010primme,wu2017primme_svds} to  efficiently compute the eigenvectors of the corresponding graph Laplacian $\bL$ without explicit formulation. As a result, the computational complexity of the eigendecomposition is reduced from $O(KN^2m)$ to $O(KNRm)$, where $m$ is the number of iterations of the underlying eigensolver.  

\item We extend existing analysis of random features to SC from the optimization perspective \cite{wu2016revisiting}, showing a number {$R = \Omega(1/(\kappa \epsilon))$} of RB features suffices for the uniform convergence to $\epsilon$ precision of the exact SC.

\item In our extensive experiments on 8 benchmark datasets evaluated by 4 different clustering metrics, the proposed technique either outperforms or matches the state-of-the-art methods in both accuracy and computational time. 

\item We corroborate the scalability of our method under two cases: varied number of data samples $N$ and varied number of RB features $R$. In both cases, our method exhibits linear scalability with respect to $N$ or $R$.

\end{enumerate}

\section{Spectral Clustering and Random Binning}
We briefly introduce the SC algorithms and then illustrate RB, an important building block to our proposed method. Here are some notations we will use throughout the paper.

\subsection{Spectral Clustering}
Given a set of $N$ data points $[\bx_1, \ldots, \bx_N] = \bX^{N \times d}$, with  $\bx_i \in \R^d$, the SC algorithm constructs a similarity matrix $\bW \in \R^{N \times N}$ representing an affinity matrix $G = (V,E)$, where the node $\bv_i \in V$ represents the data point $\bx_i$ and the edge $E_{ij}$ represents the similarity between  $\bx_i$ and $\bx_j$. 

The goal of SC is to use $\bW$ to partition $\bx_1, \ldots, \bx_N$ into $K$ clusters. There are several variants of SC. Without lose of generality, we consider the widely used \emph{normalized spectral clustering} \cite{ng2002spectral}. To fully utilize complete similarity information, we consider a fully connected graph instead of a KNN-based graph for SC. An example similarity (kernel) function is the Gaussian Kernel:
\begin{equation}
    k(\bx_i,\bx_j) = \text{exp}\Big(-\frac{\|\bx_i - \bx_j\|^2}{2\sigma^2}\Big)
\end{equation}
where $\sigma$ is the bandwidth parameter. The normalized graph Laplacian matrix $\bL$ is defined as: 
\begin{equation}
    \bL = \bD^{-1/2}(\bD-\bW)\bD^{-1/2} = \bI - \bD^{-1/2}\bW\bD^{-1/2}
\end{equation}
where $\bD\in\R^{N \times N}$ is the degree matrix with each diagonal element $\bD_{ii} = \sum_{j=1}^N \bW_{ij}$. The objective function of normalized SC can be defined as \cite{shi2000normalized}:
\begin{align}
\label{eq:sc_trace_min}
    \min_{\bU \in \R^{N \times K},\bU^T \bU=\bI}~\tr (\bU^T \bL \bU),
\end{align}
where $\tr(\cdot)$ denotes the matrix trace,
$\bI$ denotes the identity matrix, and the constraint $\bU^T \bU=\bI$ implies orthonormality on the columns of $\bU$. We further denote $\bU^*\in \R^{N \times K}$ as the optimizer of the minimization problem in (\ref{eq:sc_trace_min}), where the columns of $\bU^*$ are the $K$ smallest eigenvectors of $\bL$. Finally, the K-means method is applied on the rows of $\bU^*$ to obtain the clusters. The high computational costs of the similarity matrix $O(N^2d)$ and the eigendecomposition $O(KN^2)$ make SC   hardly scalable to large-scale problems. 

\subsection{Random Binning Features}
RB features are first introduced in \cite{rahimi2008random} and rediscovered in \cite{wu2016revisiting} to yield a faster convergence compared to other \emph{Random Features} methods for scaling up large-scale kernel machine. It considers a feature map of the form
\begin{equation}\label{RB_feature_map}
k(\bx_1,\bx_2)=\int_{\omegab} p(\omegab) \phib_{B_{\omegab}}(\bx_1)^T\phib_{B_{\omegab}}(\bx_2) \;d\omegab
\end{equation}
where a set of bins $B_{\omegab}$ defines a random grid that are determined by $\omegab=(\omega_1,u_1,...,\omega_d,u_d)$ drawn from a distribution $p(\omegab)$, and $(\omega_i,u_i)$ represents \emph{width} and \emph{bias} in the $i$-th dimension of a grid. Then for any bin $\bb \in B_{\omegab}$, the feature vector $\phib_{B_{\omegab}}(\bx)$ has
$$
\phib_{\bb}(\bx_i)=1, \;\textit{if}\; \bb = (\lfloor \frac{\bx_{i}^{(1)}-u_1}{\omega_1}\rfloor, ..., \lfloor \frac{\bx_{i}^{(d)}-u_d}{\omega_d}\rfloor)
$$ 
if the data point $\bx_i$ locates in the bin $\bb \in B_{\omegab}$. Since a data point can only locate in one bin, $\phib_{\bb}(\bx_i)=0$ for any other bins. 
Note for each grid $B_{\omegab}$ the number of bins $|B_{\omegab}|$ is countably infinite, therefore $\bphi_{B_{\omegab}}(\bx)$ has infinite dimensions but only $1$ non-zero entry (at the bin $\bx$ lies in). 
Figure \ref{fig:RB_gen} illustrates an RB example when the data dimension $d=2$.
When two data points $\bx_1$, $\bx_2$ fall in the same bin, the \emph{collision probability} for this to happen is proportional to the kernel value $k(\bx_1,\bx_2)$. Note that for a given grid multiple non-empty bins (features) can be produced and thus RB essentially generates a large sparse binary matrix (for more details, please refer to \cite{wu2016revisiting}). 

In practice, in order to obtain a good kernel approximation matrix $\bZ$, a simple Monte Carlo method is often leveraged to approximate \eqref{RB_feature_map} by averaging over $R$ grids $\{B_{\omegab_i}\}_{i=1}^R$, where each grid's parameter $\omegab_i$ is drawn from $p(\omegab)$. Algorithm \ref{alg:RB} summarizes the procedure for generating a number $R$ of RB features from the original dataset $\bX$. The resulting feature matrix $\bZ \in \R^{N \times D}$, where $D$ is determined jointly by both the number of grids $R$ and the kernel parameter $\sigma$. However, for each row the number of nonzero entries $nnz(\bZ(i,:)) = R$ and thus the total number of $nnz(\bZ) = NR$, which is the same as other random feature methods \cite{wu2016revisiting}.

\begin{figure}
    \centering
    \includegraphics[scale=0.28]{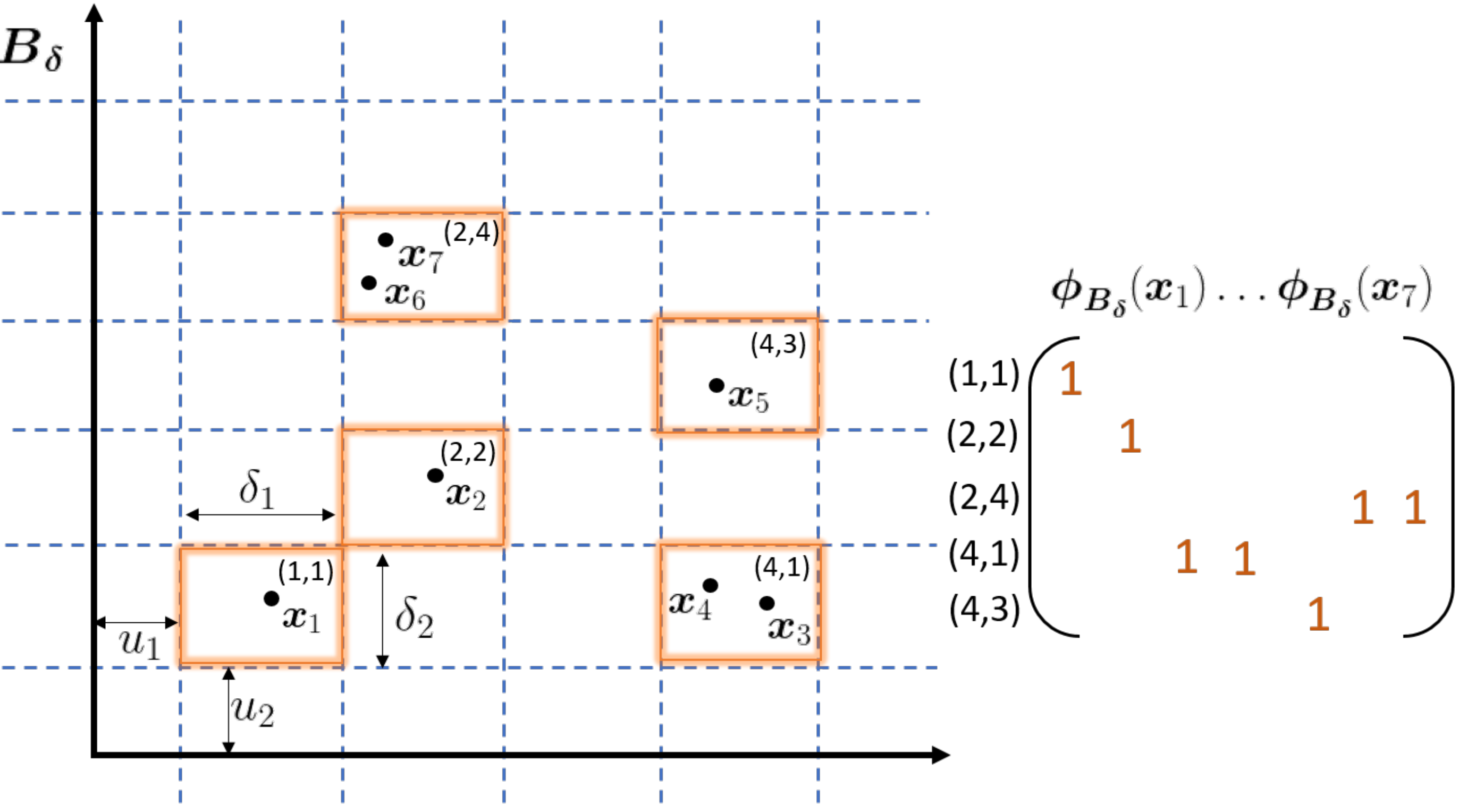}
    \caption{Example of generating  RB features when $d=2$.}
    \label{fig:RB_gen}
\end{figure}

\begin{algorithm}[t]
    \caption{ RB Features Generation}
    \label{alg:RB}
    \begin{algorithmic}[1]
    \STATEx {\bf Input:}  Given a kernel function $k(\bx_i,\bx_j)=\prod_{l=1}^d k_l(|x_{i}^{(l)}-x_{j}^{(l)}|)$. Let $p_j(\omegab) \propto \omegab k_j''(\omegab)$ be a distribution over $\omegab$.
    \STATEx {\bf Output:} RB feature matrix $\bZ^{N \times D}$ for raw data $\bX$
    \FOR {$j = 1, \ldots, R$}
        \STATE Draw $\omega_{i}$ from $p_j(\omegab)$ and $u_{i}\in [0,\omega_{i}]$, for $\forall i\in[d]$ 
        \STATE Compute feature values $\bz_{j}(\bx_i)$ as the indicator vector of bin index $(\lfloor\frac{\bx_{i}^{(1)}-u_1}{\omega_1}\rfloor, ..., 
        \lfloor\frac{\bx_{i}^{(d)}-u_d}{\omega_d}\rfloor )$, for $\forall i\in[N]$.
    \ENDFOR
    \STATE $\bZ_{i,:}=\frac{1}{\sqrt{R}}[\bz_1(\bx_i);...;\bz_D(\bx_i)]$, for $\forall i\in[N]$
    \end{algorithmic}
\end{algorithm}

\section{Scalable Spectral Clustering Using RB Features} \label{sec: scalable sc_rb}
In this section, we introduce our proposed scalable SC method, called SC\_RB, based on RB and a state-of-the-art sparse eigensolver (PRIMME) to effectively tackle the two computational bottlenecks: 1) Pairwise graph construction; and 2) Eigendecomposition of the graph Laplacian matrix.

\subsection{Pairwise graph construction}
The first step of SC is to build a similarity graph. A fully connected graph entailing complete similarity information of the original data offers great flexibility in the underlying kernel function to define the affinities between data points. However, constructing a pairwise graph essentially computes a similarity (kernel) matrix between each pair of samples, which is computationally intensive due to its quadratic complexity $O(N^2d)$. Thus, we propose to use RB features $\bZ$ to approximate the pairwise kernel matrix $\bW$, resulting in the following approximate spectral clustering objective:
\begin{equation}\label{eq:sc_rb}
\begin{aligned}
\min_{\bU \in \R^{N \times K},\bU^T \bU=\bI}~\tr(\bU^T\widehat{\bL}\bU)
\end{aligned}
\end{equation}
where 
$$
\widehat{\bL}:=\bI-\widehat{\bD}^{-1/2}\widehat{\bW}\widehat{\bD}^{-1/2}=\bI-\widehat{\bD}^{-1/2}\bZ\bZ^T\widehat{\bD}^{-1/2}
$$
and $\bZ$ is a large sparse $N\times D$ matrix generated from Algorithm \ref{alg:RB}. To apply RB to SC, it is necessary to compute the degree matrix $\widehat{\bD}$, the row sum of $\widehat{\bW}$. Luckily, we can compute it without explicitly computing $\bZ\bZ^T$ since
\begin{equation}\label{eq:sc_rb_computeD}
    \widehat{\bD} = \textnormal{diag}(\widehat{\bW}\1) = \textnormal{diag}(\bZ (\bZ^T \1))
\end{equation} 
where $\textnormal{diag}(\bx)$ is a diagonal matrix with the vector $\bx$ on its main diagonal, and $\1$ represents a column vector of ones. Therefore, we can simply compute $\widehat{\bD}$ by two matrix-vector multiplication without the explicit form of $\widehat{\bW}$. With $\widehat{\bD}$, we define $\widehat{\bZ} = \widehat{\bD}^{-1/2}\bZ$ and thus we approximate the graph Laplacian matrix with $\widehat{\bL} = \bI -  \widehat{\bZ}\widehat{\bZ}^T$ in linear complexity.  

\subsection{Effective eigendecomposition using PRIMME}
After constructing the pairwise graph implicitly using $\bZ$, we compute the largest left singular vectors of $\widehat{\bZ}$, which is equivalent to computing the smallest eigenvectors of $\widehat{\bL} := \bI - \widehat{\bZ}\widehat{\bZ}^T$ in Equation \eqref{eq:sc_rb}, satisfying 
\begin{equation}
\widehat{\bZ}\bv_i = \sigma_i \bu_i, \quad \quad i = 1, \ldots, K, \ \ K \ll N ,
\end{equation}
where the singular values are labeled in descending order, $\sigma_1 \geq \sigma_2 \geq \ldots \geq \sigma_K$. The matrices $ \bU = [\bu_{1},\ldots , \bu_{K}] \in \mathbb{R}^{N \times K}$ and $ \bV = [\bv_{1},\ldots , \bv_{n}] \in \mathbb{R}^{D \times K}$ are the left and right singular vectors respectively, where $\bU$ is the low-dimensional embedding associated with $K$ clusters. 

However, $\widehat{\bZ}$, a weighted RB feature matrix of $\bZ$, is a very large sparse matrix of the size $N \times D$, making it a challenging task for any standard SVD solver. Specifically, one has to resort to a powerful iterative sparser eigensolver that is capable to handle two difficulties for large-scale matrix: 1) slow convergence of eigenvalues when the eigenvalue gaps are not well separated, a common case when $N$ is large; 2) low memory footprint yet near-optimal convergence for seeking a small number of eigenpairs. 

To overcome these challenges, we leverage current state-of-the-art eigenvalue and SVD solver \cite{stathopoulos2010primme,wu2017primme_svds}, named PReconditioned Iterative MultiMethod Eigensolver (PRIMME). It implements two near-optimal eigenmethods GD+K and JDQMR that are default methods for seeking a small portion of extreme eigenpairs under limited memory. Unlike Lanczos methods (like Matlab svds function), these eigenmethods are in the classes of Generalized Davidson, which enjoy benefits for advanced subspace restarting and preconditioning techniques to accelerate the convergence.   

Once the left singular vectors $\bU$ are obtained, following  \cite{ng2002spectral}, we obtain $\widehat{\bU}$ by normalizing each row of $\bU$ to unit norm. Then K-means method is applied to the rows of $\widehat{\bU}$ to obtain the final $K$ clusters and the binary membership matrix $\bM \in \{0,1\}^{N \times K} $. 

\begin{algorithm}[t]
    \caption{ Scalable SC method based on RB}
    \label{alg:sc_rb}
    \begin{algorithmic}[1]
    \STATEx {\bf Input:}  Data matrix $\bX$,  number of clusters $K$, number of girds $R$, kernel parameter $\sigma$.
    \STATEx {\bf Output:} K clusters and membership matrix $\bM$
    \STATE Construct a fully connected graph using a sparse feature matrix $\bZ \in \R^{N \times D}$ generated by RB using Algorithm \ref{alg:RB}.
    \STATE Compute degree matrix $\widehat{\bD}$ using Equation \ref{eq:sc_rb_computeD} and obtain $\widehat{\bZ} = \widehat{\bD}^{-1/2}\bZ$ using Equation \ref{eq:sc_rb}.
    \STATE Compute $K$ largest left singular vectors $\bU$ of $\widehat{\bZ}$ using state-of-the-art iterative sparse SVD solver  (e.g., PRIMME).
    \STATE Obtain the matrix $\widehat{\bU}$ from $\bU$ by row normalization.
    \STATE Cluster the rows of $\widehat{\bU}$ into $K$ clusters using K-means and obtain the corresponding membership matrix $\bM$.
    \end{algorithmic}
\end{algorithm}

\textbf{Computation analysis.} Algorithm \ref{alg:sc_rb} summarizes the procedure for the proposed scalable SC method based on RB and PRIMME. Using these two important building blocks, the computational complexity has been substantially reduced from $O(N^2d)$ to $O(NRd + NR)$ for computing the feature matrix $\widehat{\bZ}$ from RB in pairwise graph construction, and from at least $O(KN^2m)$ to $O(KNRm)$ for the subsequent SVD computation, where $m$ is the number of iterations of the underlying SVD solver. At the same time, the memory consumption has been reduced from $O(N^2)$ to $O(NR)$. In addition to these two key steps, the final K-means also takes $O(NK^2t)$, where $t$ is the number of iterations of K-means. Therefore, the total computational complexity and memory consumption are $O(NRd + NKRm + NK^2t)$ and $O(NR)$. The linear complexity in the number $N$ of data points render SC scalable to large-scale problems.

\section{Theoretical Analysis}
The convergence of Random Feature approximation has been studied since it was first proposed in \cite{rahimi2008random}, where a sampling approximation analysis was employed to show the convergence of the approximation to exact kernel. Such analysis was adopted by most of its follow-up works. More recently, a new approach of analysis based on infinite-dimensional optimization is proposed in \cite{yen2014sparse}, which achieves a faster convergence rate than the previous approach, and was employed further in \cite{wu2016revisiting} to explain the superior convergence of \emph{Random Binning Features} than other types of random features in the context of classification.

Here we further adapt the analysis in \cite{wu2016revisiting} to study the convergence of Spectral Clustering under RB approximation. We first recall the well-known connection between \emph{SC} and \emph{kernel $K$-means} \cite{dhillon2004kernel}, stating the equivalence of \eqref{eq:sc_trace_min} to the following objective
\begin{equation}\label{kernel_K-means}
\begin{aligned}
\min_{\bU \in \mathbb{R}^{N \times K}:\bU^T\bU=\bI}\min_{\cM} && \frac{1}{2N} \sum_{i=1}^N \|\bphi(\bx_i)-\cM \bU_{i,:}^T\|^2
\end{aligned}
\end{equation}
where $\bphi(.)$ is a possibly infinite-dimensional feature map in \eqref{RB_feature_map} from the normalized kernel, $\cM$ is the matrix of \emph{means} with $k$ columns, each of which has the same dimension to the feature map $\bphi(.)$. Dropping constants that are neither related to $\bU$ nor related to $\cM$, the objective \eqref{kernel_K-means} becomes
\begin{equation}\label{kernel_output}
f(\bU,\cM):=\frac{1}{N}\sum_{i=1}^N -\langle \bphi(\bx_i),\cM \bU_{i,:}^T\rangle + \frac{1}{2N}\|\cM\bU_{i,:}\|^2.
\end{equation}
Let $(\bUh,\cMh)$ be the clustering from the RB approximation:
\begin{equation}\label{alg_output}
(\bUh,\cMh):=\underset{\bU,\cM}{\arg \min}\;\frac{1}{N}\sum_{i=1}^N -\langle \bz(\bx_i),\cM \bU_{i,:}^T\rangle + \frac{1}{2N}\|\cM\bU_{i,:}\|^2
\end{equation}
Let $(\bU^*,\cM^*)$ be the exact minimizer of \eqref{kernel_output}. Our goal is to show that 
\begin{equation}\label{result}
f(\bUh,\cMh)\leq f(\bU^*,\cM^*) + \epsilon
\end{equation}
as long as the number of Random Binning grids satisfies $R=\Omega(\frac{1}{\kappa\epsilon})$, where $\kappa$ is an estimate of the number of non-empty bins per grid. The quantity $\kappa$ is crucial in the our analysis, as under the same computational budget, RB generates $\kappa$ more features in expectation and converges \emph{$\kappa$-times faster} than other types of random features. Note the computational cost is not $\kappa$-times more because of the sparse structure of RB---only one of $\kappa$ features is non-zero for each sample $\bx$. The formal definition of $\kappa$ is as follows.

\begin{definition}\label{def:collision_prob}
Define the collision probability of data $\bX$ on bin $b\in B_{\bdelta}$ as: 
\begin{equation}\label{collision_prob}
     \nu_{b}:=\frac{ |\{n\in[N] \;|\; \phi_{b}=1 \}| }{N}.
\end{equation}
Let $\nu_{\bdelta} := \max_{b\in B_{\bdelta}} \nu_b$ be an upper bound on \eqref{collision_prob}, and $\kappa_{\bdelta}:=1/\nu_{\bdelta}$ be a lower bound on the number of non-empty bins of grid $B_{\bdelta}$. Then
\begin{equation}\label{expected_col_prob}
    \kappa:= E_{\bdelta}[\kappa_{\bdelta}] \geq 1
\end{equation}
is the expected number of non-empty bins.
\end{definition}

The proof of \eqref{result} contains two parts. In the first part, we show that $f(\bU,\cMh)\leq f(\bU,\cM^*)+\epsilon$ for any given $\bU$. This is obtained from the insight that the RB approximation $\bz(\bx)$ (from Algorithm \ref{alg:RB}) is a subset of coordinates from the feature map $\bphi(\bx)$. Therefore, $\cMh$ can be interpreted as a solution obtained from $R$ iterations of \emph{Randomized Block Coordinate Descent} on $f(\bU,\cM)$ w.r.t. $\cM$, which results in $R$ non-zero blocks of rows in $\cM$.

\begin{theorem}\label{thm:RBconverge}
Let $R$ be the number of grids generated by Algorithm \ref{alg:RB}. For any given $\bU$, let $\cM^*$ and $\cMh$ be the minimizers of \eqref{kernel_K-means} and \eqref{alg_output} respectively. We have
\begin{equation}\label{eq:RBconvergence}
E[f(\bU,\cMh)] - f(\bU,\cM^*) \leq \frac{\|\cM^*\|_F^2}{\kappa (R-c)}
\end{equation}
for $R>c$, where $c$ is a small constant.
\end{theorem}
\begin{proof}
Let $\bPhi:=[\bphi(\bx_1),...,\bphi(\bx_N)]$. Given $\bU$ that satisfies $\bU^T\bU=\bI$, the objective $f(\bU,\cM)$ can be written as
\begin{align*}
f(\bU,\cM)&=\frac{-1}{N}\langle\bPhi,\cM\bU^T\rangle + \frac{1}{2N} \tr(\cM\bU^T\bU\cM^T)\\
&= \frac{-1}{N}\langle \bPhi \bU ,\cM\rangle + \frac{1}{2N}\|\cM\|_F^2\\
&=\sum_{k=1}^K g(\bU_{:,k},\cM_{:,k}),
\end{align*}
where $g(\bU_{:,k},\cM_{:,k})$ is defined as
\begin{equation}\label{ERM}
\frac{1}{N}\sum_{i=1}^N -\langle \bphi(\bx_i)u_{ik},\cM_{:,k}\rangle+\frac{1}{2N}\|M_{:,k}\|^2.
\end{equation}
In other words, given $\bU$, $f(\bU,\cM)$ can be separated as $K$ independent subproblems, each solving a column of $\cM$. Let the first term of \eqref{ERM} be the loss function and the second term be the regularizer. Then \eqref{ERM} satisfies the form of a convex, smooth empirical loss minimization problem studied in \cite{wu2016revisiting}. By Theorem 1 of \cite{wu2016revisiting}, the minimizer of \eqref{ERM} satisfies
\begin{equation}\label{result_per_k}
E[g(\bU_{:,k},\cMh_{:,k})]-g(\bU_{:,k},\cM_{:,k}^*) \leq \frac{\|\cM^*_{:,k}\|^2}{\kappa(R-c_k)}
\end{equation}
with $c_k:=\lceil\frac{2\kappa( g(\bU_{:,k},\0)-g(\bU_{:,k},\cM^*_{:,k}))}{\|\cM_{:,k}^*\|^2}\rceil$. Summing \eqref{result_per_k} over $k=1...K$, we have
$$
E[f(\bU,\cMh)] - f(\bU,\cM^*) \leq \frac{\|\cM^*\|_F^2}{\kappa (R-c)}
$$
where $c:=\max_{k=1}^K c_k$.
\end{proof}

Theorem \ref{thm:RBconverge} implies that $f(\bU,\cMh)\leq f(\bU,\cM^*)+\epsilon$ for
\begin{equation}\label{tmp3}
R_{RB}\geq \frac{\|\cM^*\|_F^2}{\kappa\epsilon}+c.
\end{equation}
As noted by the earlier work \cite{wu2016revisiting}, this convergence rate is $\kappa$ times faster than that of other Random Features under the same analysis framework. More specifically, if applying Theorem 2 of \cite{yen2014sparse} instead of Theorem 1 of \cite{wu2016revisiting} in the proof of Theorem \ref{thm:RBconverge}, one would have obtained a $\kappa$-times slower convergence rate for a general Random Feature method that generates a single feature at a time, which requires
$$
R_{RF}\geq \frac{\|\cM^*\|_F^2}{\epsilon}+c'
$$
number of features to guarantee an $\epsilon$ suboptimality. This is owing to RB's ability to generate a block of $\kappa$ expected number of features at a time. 

In the second part of the proof, we show that the spectral clustering $\bUh$ obtained from the RB approximation converges to $\bU^*$ in the objective.

\begin{theorem}\label{thm:SCconverge}
Let $R$ be the number of grids generated by Algorithm \ref{alg:RB}, and let $\bU^*$, $\bUh$ be the spectral clusterings obtained from \eqref{kernel_K-means}, \eqref{alg_output} respectively. We have
\begin{equation}\label{eq:RBconverge}
E[f(\bUh,\cMh)] - f(\bU^*,\cM^*) \leq \epsilon
\end{equation}
for
$$
R\geq \frac{\|\cM^*\|_F^2}{\kappa\epsilon}+c
$$
where $c$ is a small constant (defined in Theorem \ref{thm:RBconverge}).
\end{theorem}
\begin{proof}
Note the problem \eqref{alg_output} can be solved with global optimal guarantee by finding minimum eigenvalues and eigenvectors of \eqref{eq:sc_rb}. Therefore, let $\cMh(\bUh)$, $\cMh(\bU^*)$ be the minimizers of \eqref{alg_output} under $\bUh$, $\bU^*$ respectively. We have
\begin{equation}\label{thm2_tmp1}
f(\bUh,\cMh(\bUh)) \leq f(\bU^*,\cMh(\bU^*))
\end{equation}
by the optimality of $(\bUh,\cMh(\bUh))$ under the approximate feature map $\bz(\bx)$. In addition, from Theorem \ref{thm:RBconverge} we have
\begin{equation}\label{thm2_tmp2}
f(\bU^*,\cMh(\bU^*))-f(\bU^*,\cM^*)\leq \frac{\|\cM^*\|_F^2}{\kappa (R-c)}
\end{equation}
for $R>c$. Combining \eqref{thm2_tmp1} and \eqref{thm2_tmp2} leads to the result.
\end{proof}

\section{Experiments}
We conduct experiments to demonstrate the effectiveness and efficiency of the proposed method, and compare against 8 baselines on 8 benchmarks. Our code \footnote{https://github.com/IBM/SpectralClustering\_RandomBinning} is implemented in Matlab and we use C Mex functions for computationally expensive components of RB \footnote{https://github.com/teddylfwu/RandomBinning} and of PRIMME eigensolver \footnote{https://github.com/primme/primme}. All computations are carried out on a linux machine with Intel Xeon CPU at 3.3GHz for a total of 16 cores and 500 GB main memory. 
\begin{table}[htbp]
\centering
\caption{Properties of the datasets.} 
\vspace{0mm}
\label{tb: info of datasets}
\begin{center}
    \begin{tabular}{ c c c c}
    \hline
    Name 		 & $K$: Classes & $d$: Features & $N$: Samples \\ \hline 
    pendigits 	 & 10  & 16  & 10,992 \\
    letter       & 26 & 16 & 15,500 \\
    mnist        & 10 & 780 & 70,000  \\ 
    acoustic     & 3  & 50 & 98,528  \\ 
    ijcnn1 		 & 2  & 22  & 126,701 \\ 
    cod\_rna     & 2  & 8  & 321,054  \\
    covtype-mult & 7  & 54 & 581,012 \\ 
    poker        & 10 & 10 & 1,025,010 \\ \hline
    \end{tabular}
\end{center}
\end{table}

\textbf{Datasets.} As shown in Table 1, we choose 8 datasets from LibSVM \cite{chang2011libsvm}, where 5 of them overlap with the datasets used in \cite{yan2009fast,li2016scalable,chen2011large}. We summarize them as follows:

1) \textbf{pendigits.} A collection of handwritten digit data set consisting of 250 samples from 44 writers where sampled coordination information are used to generate 16 features per sample;

2) \textbf{letter.} A collection of images for 26 capital letters in the English alphabet where 16 character image features are generated; 

3) \textbf{minst.} A popular collection of handwritten digit data set distributed by Yann LeCun, where each image is represented by a 784 dimensional vector; 

4) \textbf{acoustic.} A collection of time-series data from various sensors in the moving vehicles for measuring the acoustic modality where a 50 dimensional feature vector is generated by using FFT for each time-series; 

5) \textbf{ijcnn1.} A collection of time-series data from IJCNN 2001 Challenge, where 22 attributes are generated as a feature vector; 

6) \textbf{cod\_rna.} A collection of non-coding RNA sequences, where the total 8 features are generated by counting the frequencies of 'A', 'U', 'C' of sequences 1 and 2 as well as the length of the shorter sequence and deltaG\_total value;

7) \textbf{covtype-mult.} A collection of samples for predicting the forest cover type from cartographic variables, where the total 54 feature vector is generated for representing a sample;

8) \textbf{poker.} A collection of poker record samples where each hand consisting of five playing cards drawn from a standard deck of 52 generates a feature vector of 10 attributes.

\begin{table*}[t]
\centering
\caption{Average rank scores comparing SC\_RB against others methods using $R=1024$.}
\vspace{0mm}
\label{tb:ave_rank_alldata}
\newcommand{\Bd}[1]{\textbf{#1}}
\begin{center}
    \begin{tabular}{ c c c c c c c c c c}
    \hline
    Dataset & K-means & SC & KK\_RS & KK\_RF & SV\_RF & SC\_LSC & SC\_Nys & SC\_RF & SC\_RB \\ \hline 
    pendigits  & 3.00 & 4.75 & 2.00 & 7.75 & 8.5 & \Bd{1.00} & 4.75 & 7.25 & 5.00 \\ 
    letter	 & 8.50 & 5.75 & 5.50 & 7.50 & 4.75 & 3.25 & 4.75 & 3.75 & \Bd{1.25} \\ 
    mnist	& 5.00 & 4.25 & 5.00 & 9.00 & 8.00 & \Bd{1.00} & 3.25 & 6.75 & 2.75 \\ 
    acoustic  & 4.75 & -- & 4.25 & 6.25 & 5.75 & 3.50  & 4.75 & 5.75 & \Bd{1.00} \\ 
    ijcnn1  & 4.50 & -- & 5.75 & 2.00 & 4.00 & 6.75  & 4.75 & 7.25 & \Bd{1.00} \\ 
    cod\_rna  & 5.75 & -- & 3.50 & 5.00 & 7.75 & 5.50 & 4.00 & 2.75 & \Bd{1.75} \\ 
    covtype-mult  & 3.75 & -- & 5.25 & 5.75 & 6.50 & 2.50 & 4.75 & 5.75 & \Bd{1.75} \\ 
    poker  & 4.33 & -- & 4.00 & \Bd{3.33} & 4.67 & 5.67 & 5.00 & 4.33 & 4.67  \\ \hline
    \end{tabular}
\end{center}
\end{table*}

\begin{table*}[t]
\centering
\caption{Computational time (seconds) comparing SC\_RB against others methods using $R=1024$.}
\vspace{0mm}
\label{tb:runtime_alldata}
\newcommand{\Bd}[1]{\textbf{#1}}
\begin{center}
    \begin{tabular}{ c c c c c c c c c c}
    \hline
    Dataset & K-means & SC & KK\_RS & KK\_RF & SV\_RF & SC\_LSC & SC\_Nys & SC\_RF & SC\_RB \\ \hline 
    pendigits  & 0.8 & 25.0 & 10.7 & 10.4 & 1.0 & 7.6 & 2.5 & 1.4 & 1.8 \\ 
    letter	 & 5.9 & 171.4 & 17.1 & 36.9 & 8.9 & 27.1 & 14.6 & 10.0 & 7.7 \\ 
    mnist	& 278.1 & 2661 & 79.1 & 312.4 & 22.6 & 25.5 & 31.0 & 20.5 & 25.9 \\ 
    acoustic  & 10.2 & -- & 34.7 & 83.7 & 6.3 & 16.7  & 20.1 & 7.0 & 10.7 \\ 
    ijcnn1  & 4.2 & -- & 44.2 & 89.6 & 5.1 & 9.9  & 18.5 & 5.5 & 34.7 \\ 
    cod\_rna  & 6.7 & -- & 88.2 & 190.0 & 8.6 & 8.9 & 46.8 & 13.0 & 24.2 \\ 
    covtype-mult  & 60.7 & -- & 180.2 & 220.0 & 40.5 & 181.1 & 99.1 & 41.5 & 1593 \\ 
    poker  & 102.4 & -- & 363.1 & 5812 & 254.5 & 337.4 & 340.6 & 293.3 & 538.4  \\ \hline
    \end{tabular}
\end{center}
\end{table*}

\textbf{Baselines.} We compare against 8 random feature based SC or approximation SC methods:

1) \textbf{SC\_Nys} \cite{fowlkes2004spectral}: a fast SC method based on Nystr{\"o}m method; 

2) \textbf{SC\_LSC} \cite{chen2011large}: approximate SC for KNN-based bipartite graph between raw data and anchor points selected by K-means;

3) \textbf{SV\_RF} \cite{chitta2012efficient}: fast kernel K-means using singular vectors of the RF feature matrix (approximating similarity matrix $\bW$); 

4) \textbf{SC\_RF}: we modify SV\_RF method to become a fast SC method based on RF feature matrix (approximating Laplacian matrix $\bL$); 

5) \textbf{KK\_RF} \cite{chitta2012efficient}: another kernel K-means approximation method directly using the RF feature matrix; 

6) \textbf{KK\_RS} \cite{chitta2011approximate}: an approximate Kernel K-means by a random sampling approach; 

7) \textbf{SC} \cite{ng2002spectral}: Exact SC method; 

8) \textbf{K-means} \cite{hartigan1979algorithm}: standard K-means method applied on original dataset. 

\textbf{Evaluation metrics.} We use 4 commonly used clustering metrics for cluster quality evaluation, which has been advocated and discussed in \cite{zaki2014data}.
Let  $\{\cC_k\}_{k=1}^K$ and $\{\cC^\prime_k\}_{k=1}^{K}$ denote the $K$  clusters found by a clustering algorithm and the true cluster labels, respectively. The considered clustering metrics are: 

1) \textbf{Normalized mutual information (NMI)}: 
$$	
\textnormal{NMI}=\frac{2 \cdot I(\{\cC_k\},\{\cC^\prime_k\})}{|H(\{\cC_k\})+H(\{\cC^\prime_k\})|},
$$ 
where $I$ is the mutual information between $\{\cC_k\}_{k=1}^K$ and $\{\cC^\prime_k\}_{k=1}^{K}$, and $H$ is the entropy of clusters.

2) \textbf{Rand index (RI)}: 
$$
\textnormal{RI}=\frac{TP+TN}{TP+TN+FP+FN},
$$ 
where $TP$, $TN$, $FP$ and $FN$ represent true positive, true negative, false positive, and false negative decisions, respectively. 

3) \textbf{F-measure (FM)}: 
$$
\textnormal{FM}=\frac{1}{K} \sum_{k=1}^K \textnormal{F-measure}_k,
$$ 
where $	\textnormal{F-measure}_k= \frac{2 \cdot PREC_k \cdot RECALL_k}{PREC_k + RECALL_k}$, and $PREC_k $ and $RECALL_k$ are the precision and recall values for cluster $\cC_k$.

4) \textbf{Accuracy (Acc)}: 
$$
\textnormal{Acc}=\frac{1}{N} \sum_{i=1}^N \delta(\{\cC_k\}_{k=1}^K,\{\cC^\prime_k\}_{k=1}^{K})
$$
where $N$ is the total number of samples and the best mapping function $\delta(x,y)$ is the delta function that equals 1 if $x=y$ and equals 0 otherwise between cluster labels $\{\cC_k\}_{k=1}^K$ and the true labels $\{\cC^\prime_k\}_{k=1}^{K}$ for each sample. 

These metrics are all scaled between 0 and 1, and higher value means better clustering.

\textbf{Average rank score.} To combine multiple clustering metrics for performance evaluation of different SC methods, we adopt the methodology proposed in \cite{yang2015defining} and use the average rank score of all clustering metrics as the final performance metric. Therefore, lower average rank score means better clustering performance.

\textbf{Parameter selection.} We use the RBF kernel for all similarity (Kernel) based methods on all datasets, where the kernel parameter $\sigma$ is obtained through cross-validation within typical range [0.01 100]. All methods use the same kernel parameters to promote a fair comparison. For other parameters, we use the recommended settings in the literature.

\subsection{Clustering accuracy and computational time on all datasets}

\textbf{Setup.} We first compare against 8 aforementioned baselines in terms of both average rank score and computational time. We use the methodology proposed in \cite{yang2015defining} to compute the average ranking score among 4 different metrics NMI, RI, FM, and Acc. Although the rank $R$ has different meanings in each method but similar effects on the performance, we choose $R=1024$ for all methods to promote a fair comparison. In addition, we use PRIMME\_SVDS to accelerate SVD decomposition for all methods except SC\_LSC. For the final step of SC, we use Matlab's internal K-means function with 10 replicates. All methods use same random seeds so the difference caused by randomness is minimized. 

\textbf{Results.} Table \ref{tb:ave_rank_alldata} shows that SC\_RB consistently outperforms or matches state-of-the-art SC methods in terms of average ranking score on 5 out of 8 datasets (except pendigits, mnist, poker). The first highlight in the table is that SC\_LSC has quite good performance in the majority of datasets, especially for pendigits and mnist, owing to the sparse low rank approximation using AnchorGraph technique \cite{liu2010large}. However, we would like to point out that in SC\_LSC the similarity matrix is built on a KNN-based graph, which is essentially different from other SC methods which use a fully connected graph. This explains why SC\_LSC has even better performance than the exact SC method in these two datasets. For poker, all methods have very close numbers in four different metrics, which leads to quite similar average ranking score for all methods. Secondly, the SC type methods such as SC\_Nys, SC\_RF, and SC\_RB generally achieve better ranking scores compared to similarity-based on methods such as KK\_RF and SV\_RF. This is because the SC type methods are built on a Normalized Cuts formulation that yields a better performance guarantee in terms of the graph clustering. Finally, the improved performance of SC\_RB in the majority of datasets stems from the fact that it directly approximate a pairwise similarity matrix, which utilizes all information from the raw data. Its faster convergence allows it to retrieve more information with a similar rank $R$ compared to other methods. 

Table \ref{tb:runtime_alldata} illustrates that SC\_RB can achieve similar computational time to the other methods despite of a very large sparse matrix generated from RB, due to an important factor - near-optimal eigensolver PRIMME.
One should not be surprised that the empirical runtime of various scalable SC methods has relatively large range of differences. It is because that the constant factor in the computation complexity may vary with different datasets but the total computational costs are still bounded by $O(NRd+KNRm+NK^2t)$. This constant factor typically depends on different characteristics of various datasets and specific method. For instance, KK\_RF often needs more computational time than other methods since it needs firstly compute a dense feature matrix $Z$ of $N \times R$ size and applied K-means directly on $Z$. When $R$ is relatively large, the computation of K-means requiring $NRKt$ complexity may start dominating the total complexity, which is observed in the Table \ref{tb:runtime_alldata}. Similarly, the computational time on covtype-mult with SC\_RB is substantially heavier than those of other methods since the eigenvalues of the corresponding Laplacian matrix is very clustered making the number of iterations $m$ much more than the usual (typically 10 - 100 iterations) in other cases. Nevertheless, both complexity analysis and empirical runtime corroborate that SC\_RB is computationally as efficient as other random features based SC methods and approximation Kernel K-means methods in most of cases.

\subsection{Effects of RB on runtime and convergence}

\begin{figure}[!htb]
\centering
	  \begin{subfigure}[b]{0.23\textwidth}
      \includegraphics[width=\textwidth]{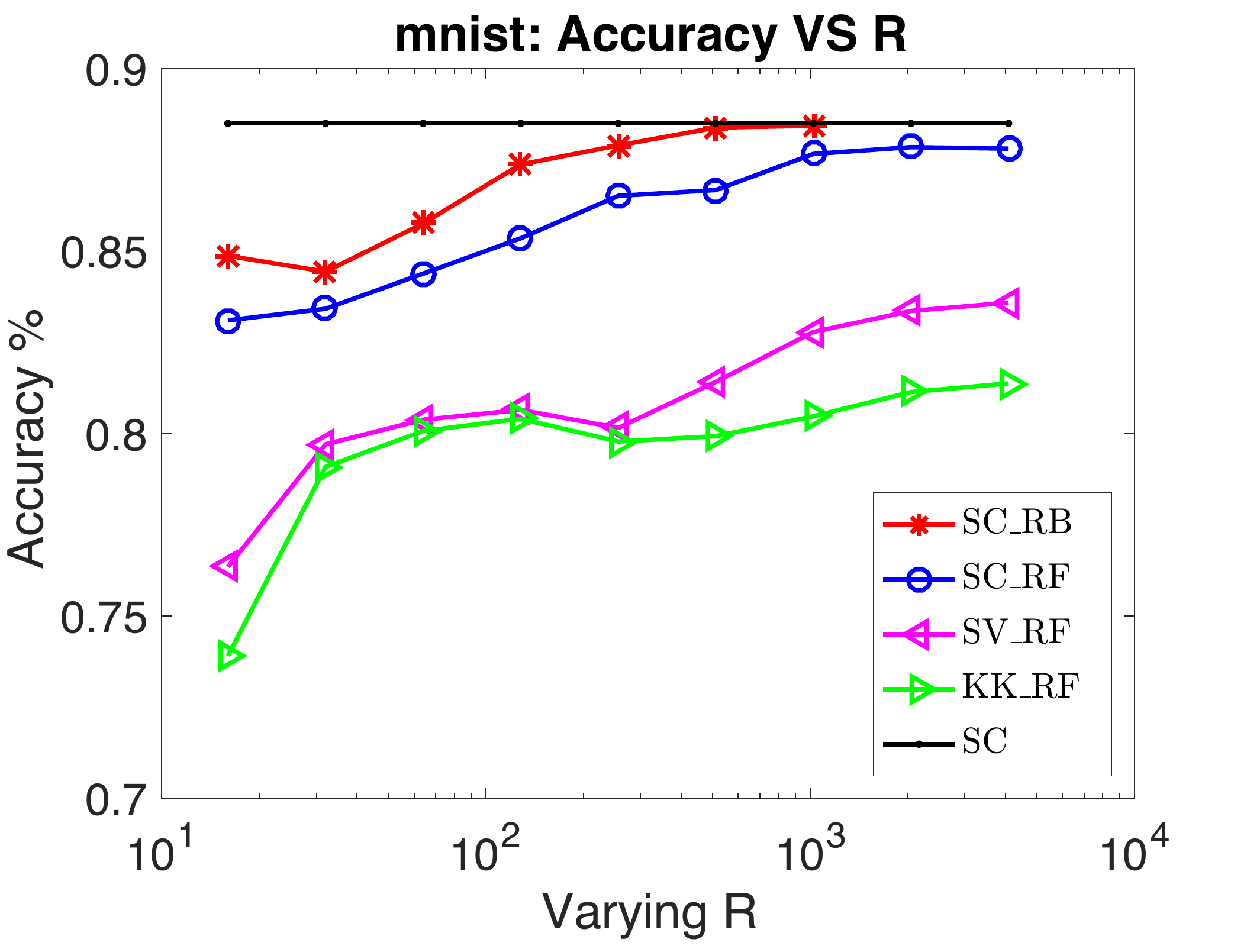}
      \caption{Accuracy (Acc)}
      \label{fig:Accu_varyingR_mnist}
      \end{subfigure}
	  \begin{subfigure}[b]{0.23\textwidth}
      \includegraphics[width=\textwidth]{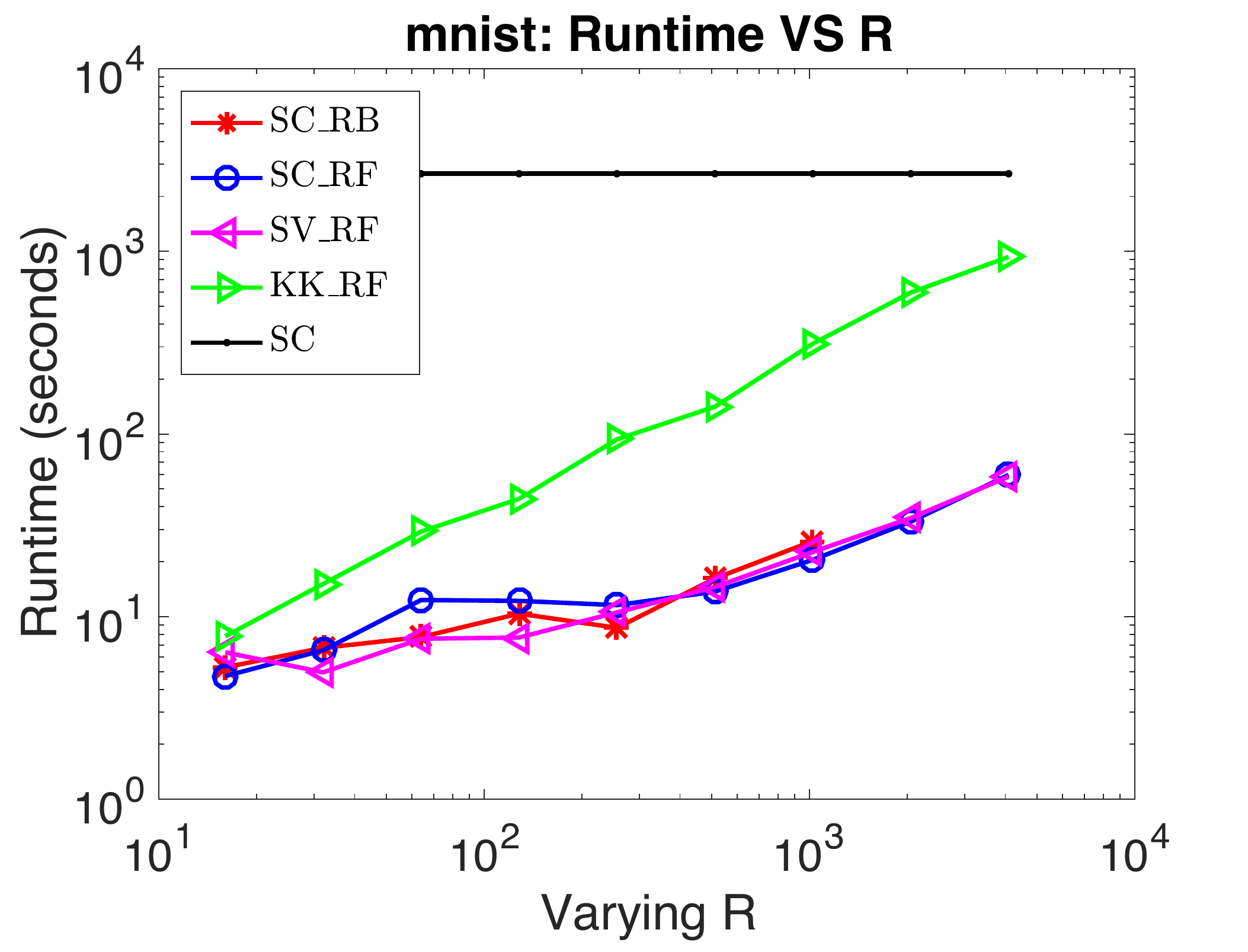}
      \caption{Runtime}
      \label{fig:Runtime_varyingR_mnist}
      \end{subfigure}
\caption{Clustering accuracy and runtime when varying $R$ on mnist for random features based SC methods.}
 \vspace{0mm}
\label{fig:Accu_runtime_varyingR_mnist}
\end{figure}

\begin{figure}[!htb]
\centering
	  \begin{subfigure}[b]{0.23\textwidth}
      \includegraphics[width=\textwidth]{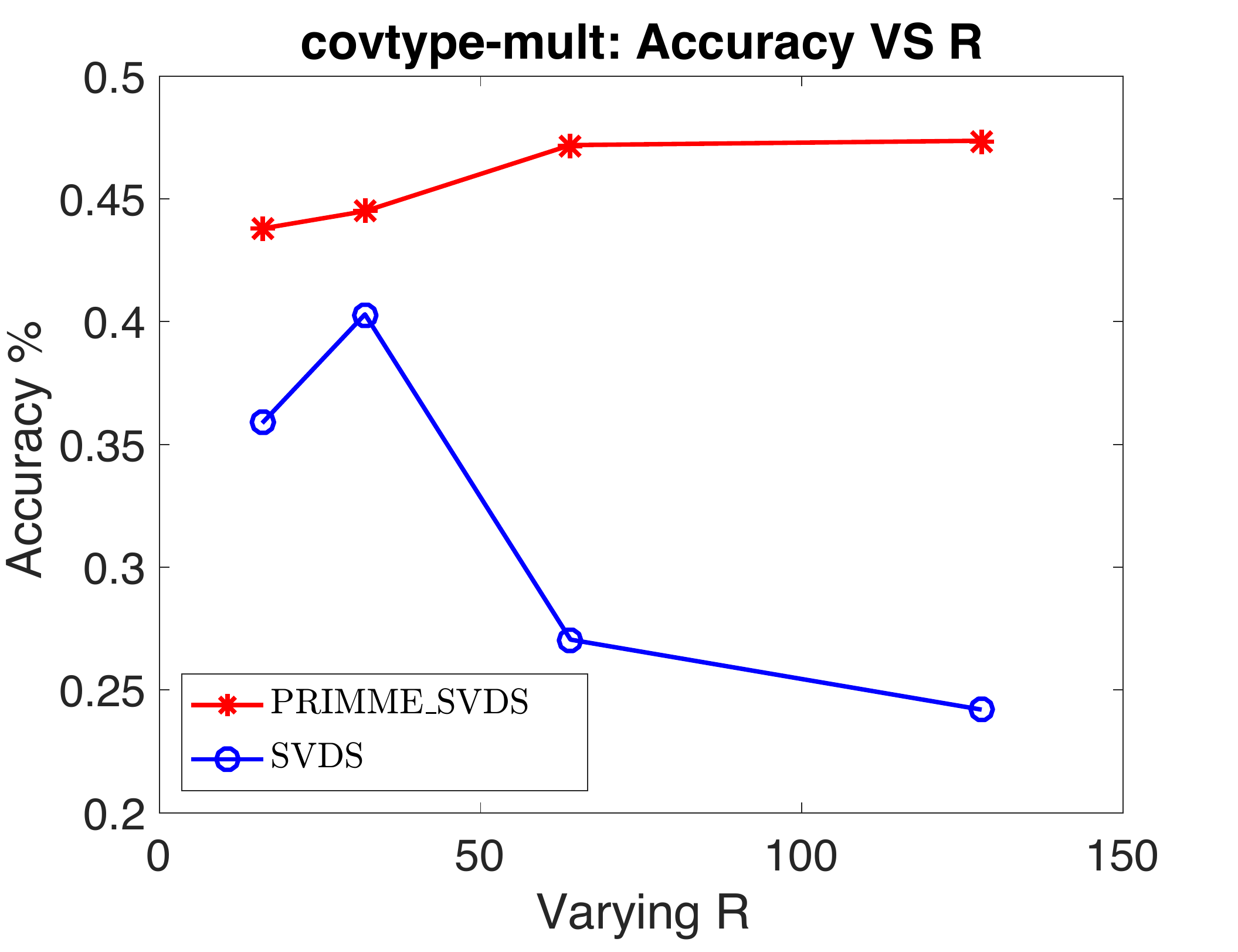}
      \caption{Accuracy (Acc)}
      \label{fig:Accu_varyingR_covtype-mult}
      \end{subfigure}
	  \begin{subfigure}[b]{0.23\textwidth}
      \includegraphics[width=\textwidth]{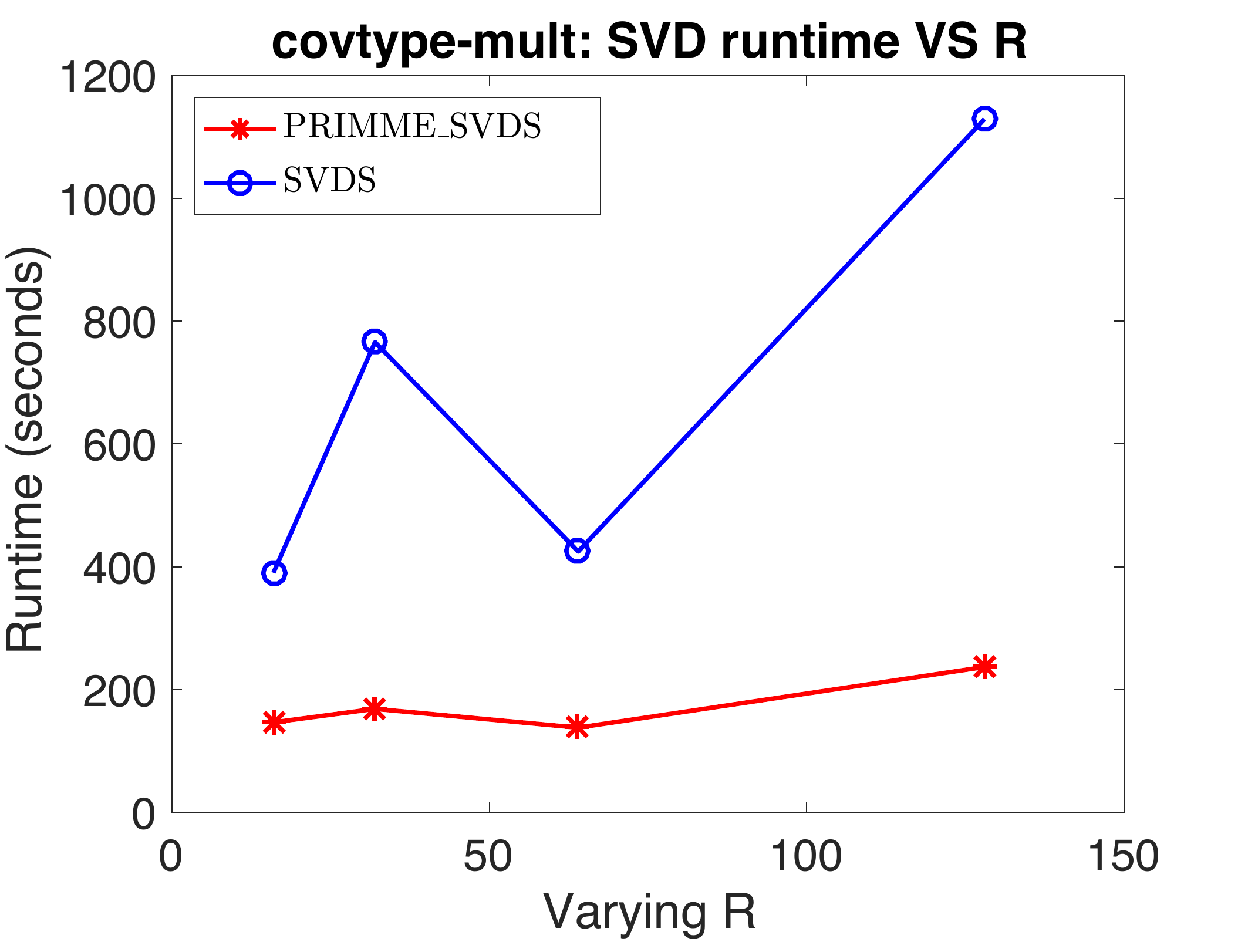}
      \caption{Runtime}
      \label{fig:Runtime_varyingR_covtype-mult}
      \end{subfigure}
\caption{Clustering accuracy and runtime when varying $R$ on covtype-mult using PRIMME\_SVDS and Matlab SVDS.}
 \vspace{0mm}
\label{fig:Accu_runtime_varyingR_covtype-mult}
\end{figure}

\textbf{Setup.} 
The first goal here is to investigate the scalability of SC\_RB over the vanilla SC method in terms of runtime while achieving the similar performance. The second goal is to study the behavior of various scalable SC and approximate Kernel K-means methods based on different random features. We limit our comparisons among two random features (RF and RB) based SC or Kernel K-means type methods.  
We choose the mnist dataset since it has been widely studied for  convergence analysis of approximation in the literature \cite{chitta2012efficient,chen2011large,li2016scalable}. We report runtime and commonly used Accuracy (Acc) as our measurement metric when varying the rank $R$ from 16 to 4096 (except SC\_RB from 16 to 1024).

\textbf{Results.}
We investigate how the performance of different methods changes when the number $R$ of random features (RF and RB) increases from 16 to 4096. 
Fig. \ref{fig:Runtime_varyingR_mnist} illustrates that despite a large sparse feature matrix generated by RB, the computational time of SC\_RB is orders of magnitudes less expensive compared to that of exact SC, and is comparable to other SC methods based on RF features. This is the desired feature of SC\_RB that it can achieve higher accuracy than other efficient SC methods without comprising the computation time. 
As shown in Fig. \ref{fig:Accu_varyingR_mnist}, we can see that the clustering accuracy (Acc) of all methods generally converge to that of exact SC but with different convergence rates. More importantly, SC\_RB yields faster convergence compared to other scalable SC methods based on RF features, which confirms our analysis in Theorem \ref{thm:SCconverge}. For instance, SC\_RB with $R=1024$ has already reached the same accuracy as the exact SC method while SC\_RF converges relatively slower to the exact SC and get close to SC with $R=4096$. Interestingly, SV\_RF and KK\_RF are not competitive in Accuracy, indicating that approximating the graph Laplacian matrix $L$ is more beneficial than these that approximating the similarity matrix $W$ in some cases.

\subsection{Effects of SVD solvers on runtime}

\begin{figure}[!htb]
\centering
	  \begin{subfigure}[b]{0.23\textwidth}
      \includegraphics[width=\textwidth]{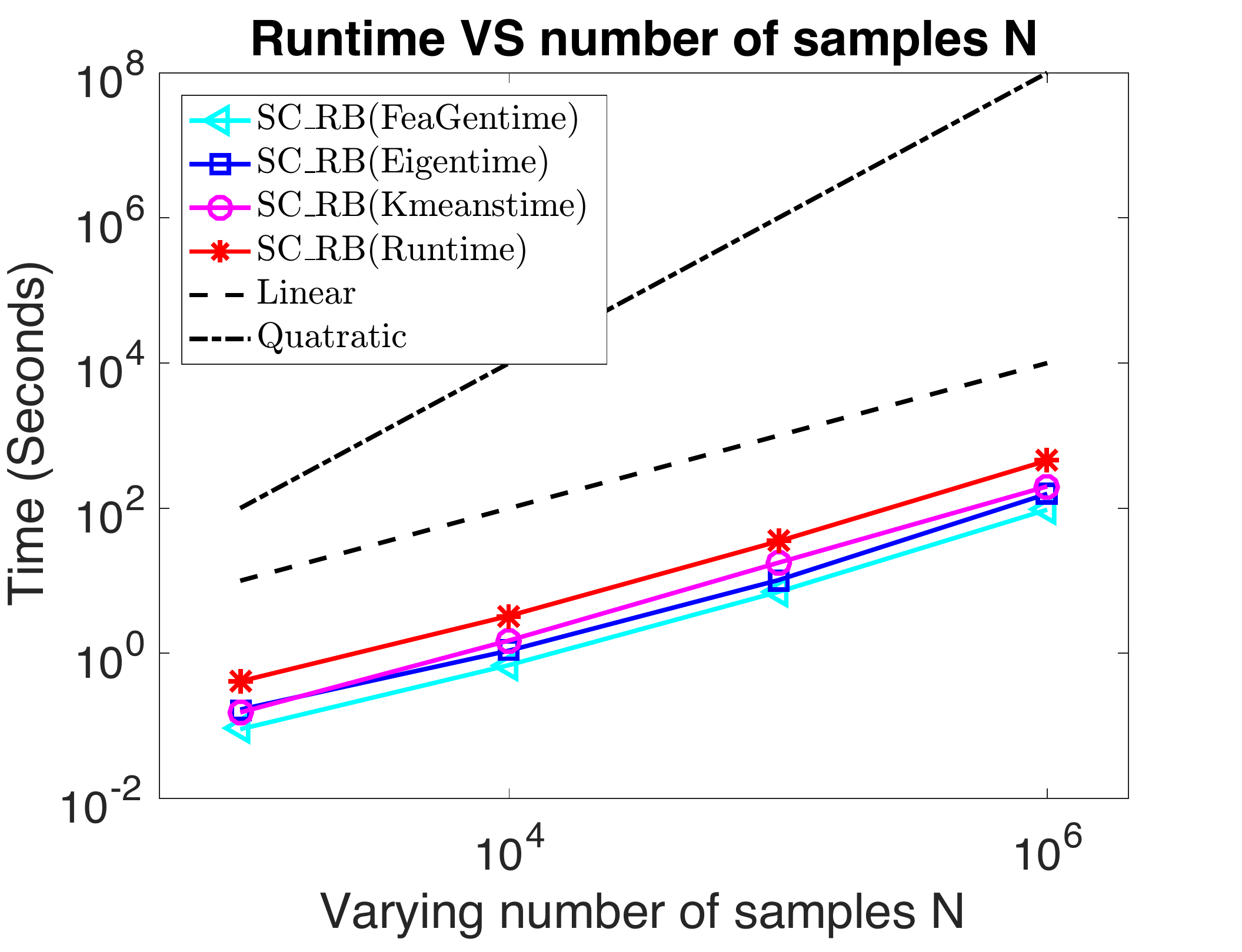}
      \caption{Size of Poker (1M)}
      \label{fig:scalability_varyingN_poker}
      \end{subfigure}
	  \begin{subfigure}[b]{0.23\textwidth}
      \includegraphics[width=\textwidth]{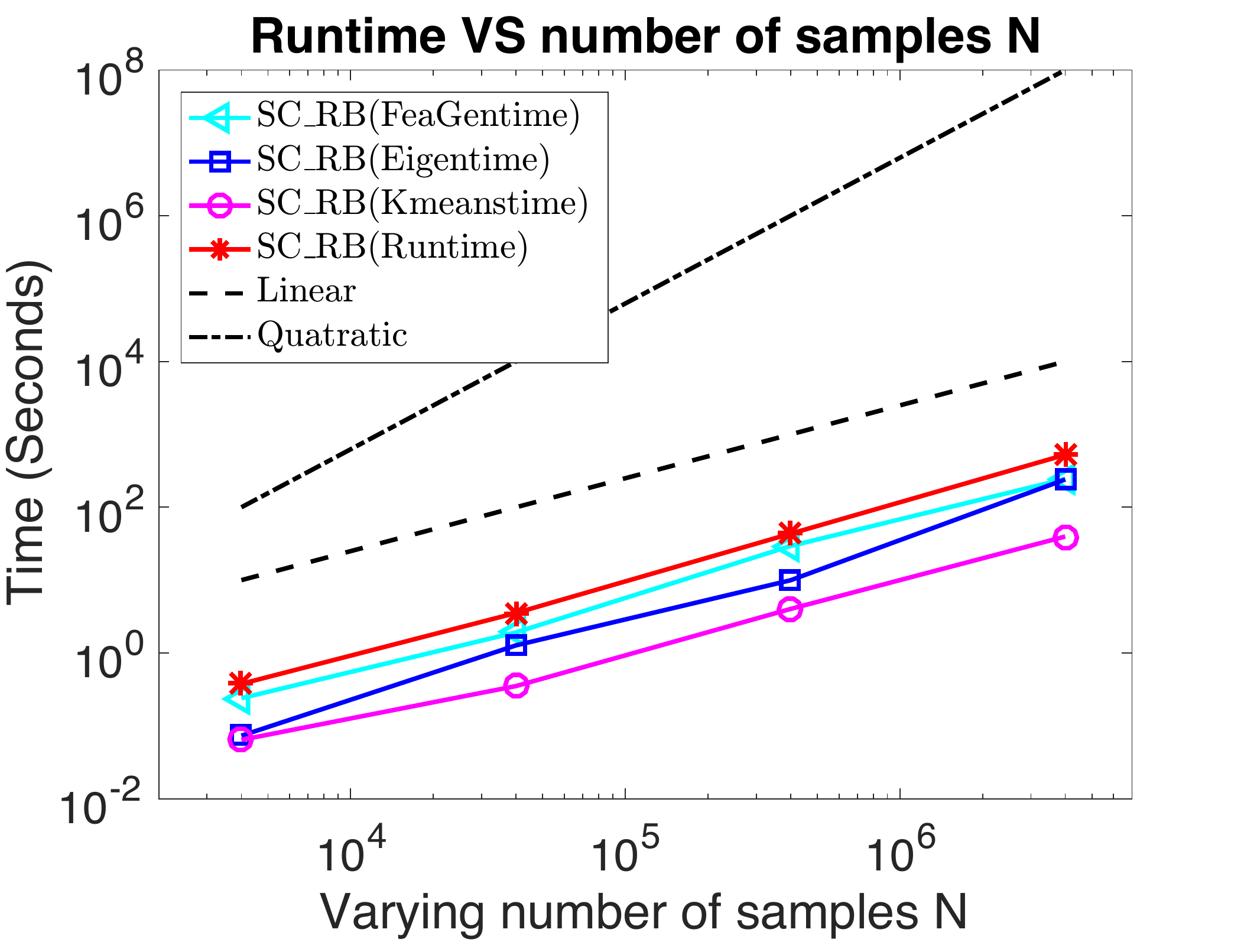}
      \caption{Size of SUSY (4M)}
      \label{fig:scalability_varyingN_susy}
      \end{subfigure}
\caption{Linear scalability of SC\_RB when varying the number of samples $N$. Linear and quadratic complexity are also plotted for easy comparisons.}
\label{fig:scalability_varyingN}
\end{figure}

\begin{figure*}[!htb]
\centering
	  \begin{subfigure}[b]{0.23\textwidth}
      \includegraphics[width=\textwidth]{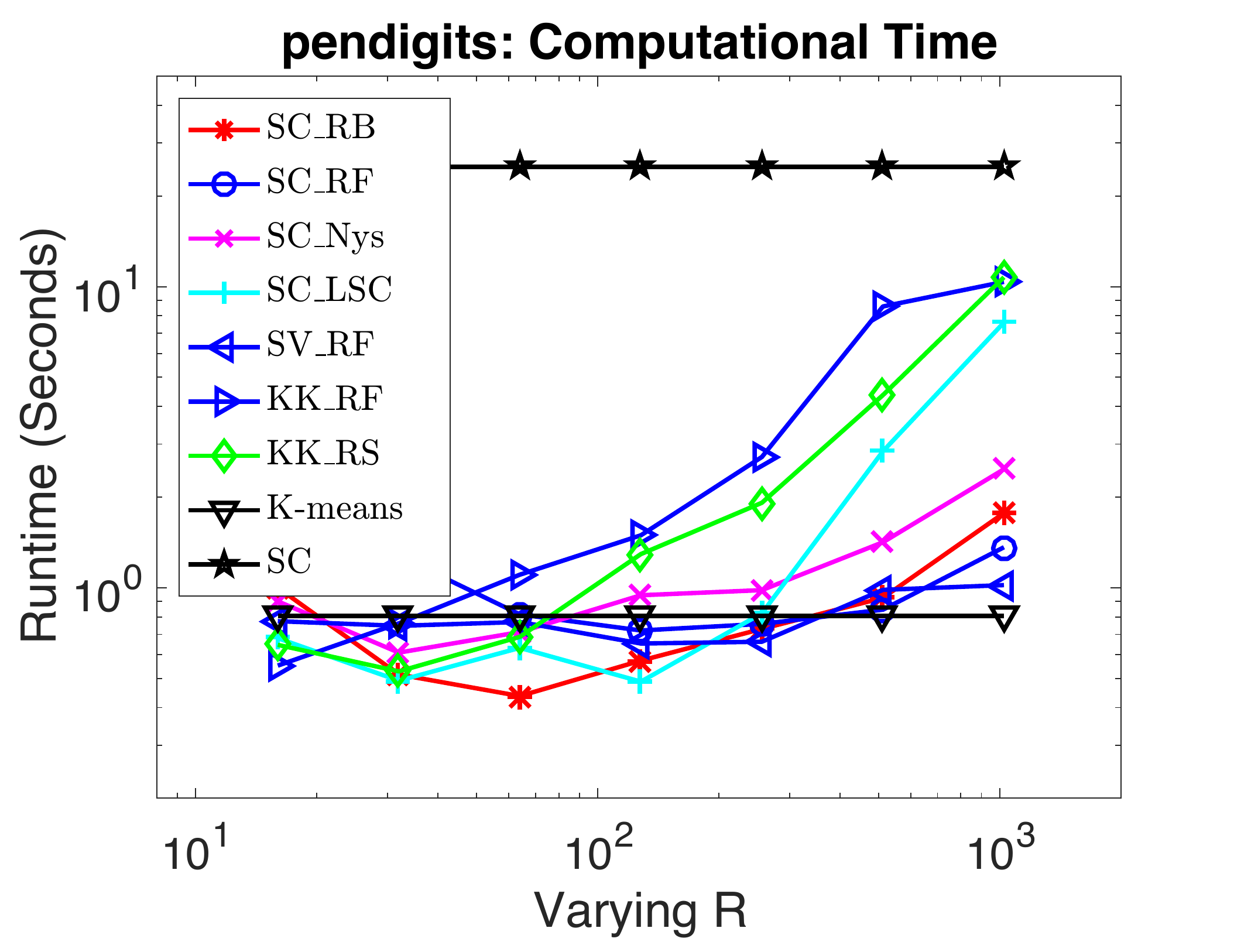}
      \caption{pendigits}
      \label{fig:scalability_varyingR_pendigits}
      \end{subfigure}
	  \begin{subfigure}[b]{0.23\textwidth}
      \includegraphics[width=\textwidth]{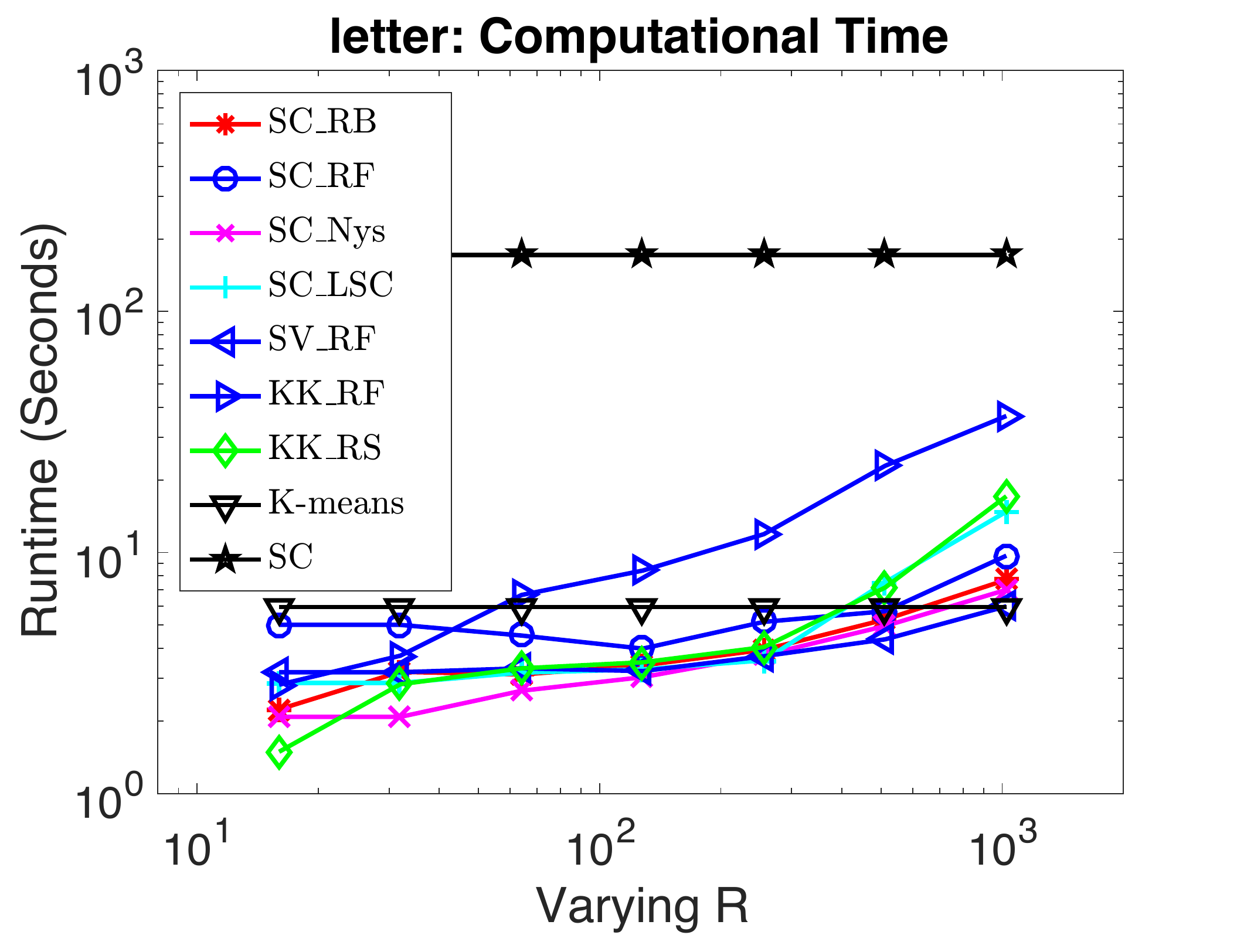}
      \caption{letter}
      \label{fig:scalability_varyingR_letter}
      \end{subfigure}
      \begin{subfigure}[b]{0.23\textwidth}
      \includegraphics[width=\textwidth]{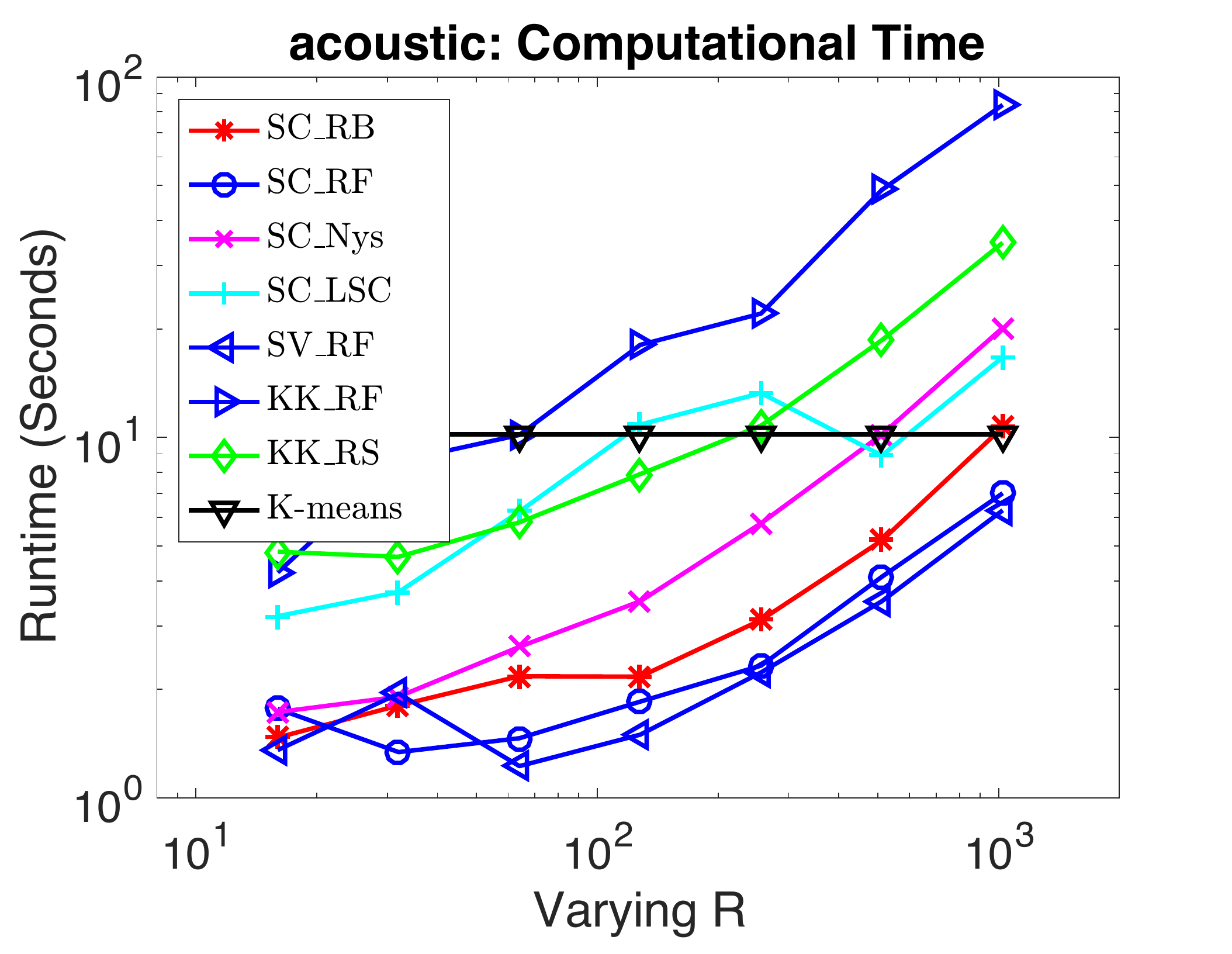}
      \caption{acoustic}
      \label{fig:scalability_varyingR_acoustic}
      \end{subfigure}
       \begin{subfigure}[b]{0.23\textwidth}
      \includegraphics[width=\textwidth]{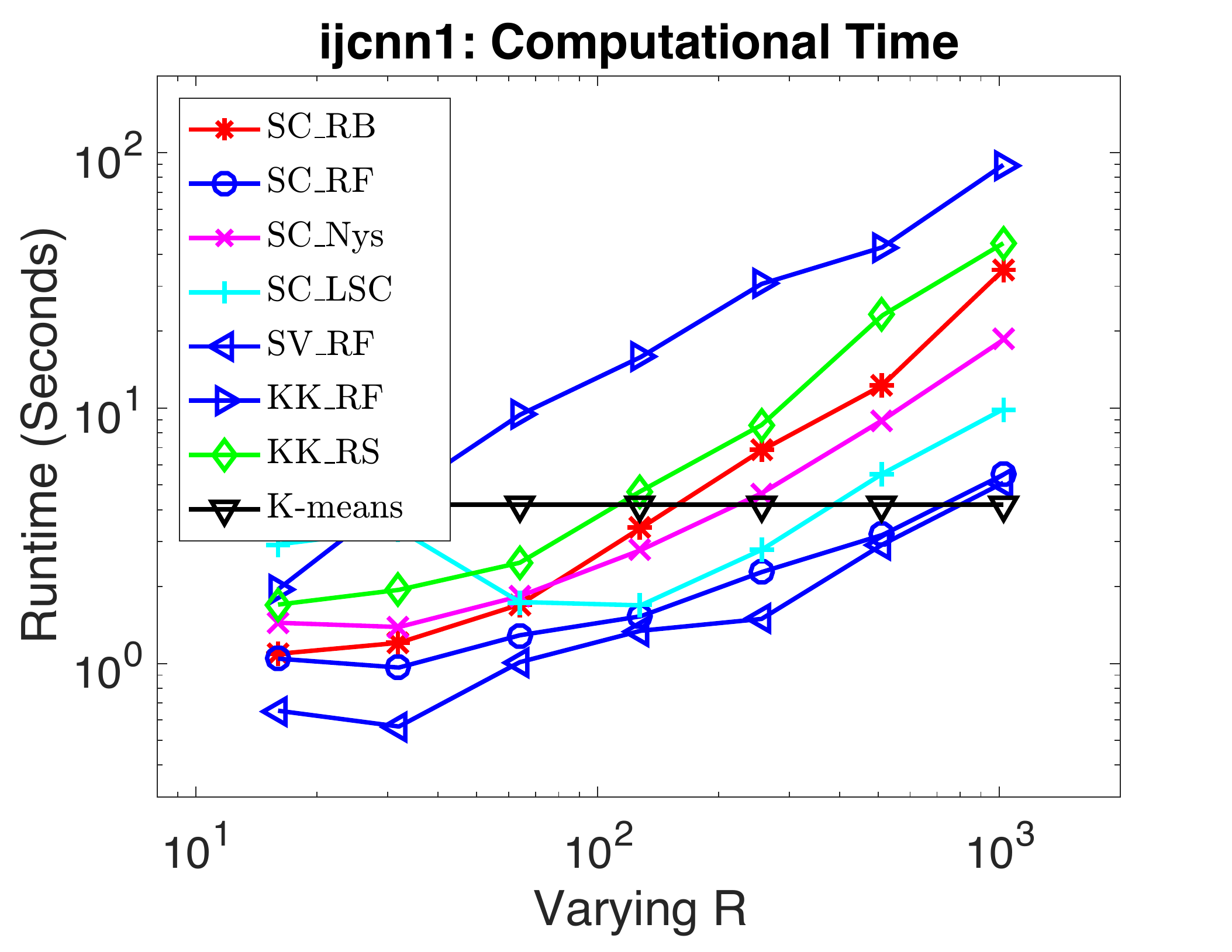}
      \caption{ijcnn1}
      \label{fig:scalability_varyingR_ijcnn1}
      \end{subfigure}
\caption{Scalability of SC\_RB and other methods on 4 datasets when varying the number of latent features $R$.}
\label{fig:scalability_varyingR}
\end{figure*}

\textbf{Setup.} We perform experiments to study the effects of various SVD solvers on runtime for SC\_RB. We choose the covtype-mult dataset due to two reasons: 1) RB generates a very large sparse matrix $\bZ$ having the size of half millions in the number of data points and tens of millions in the number of sparse features, which challenges any existing SVD solver in a single machine; 2) the convergence of iterative eigensolver largely depends on the well-separation of the desired eigenvalues. Unfortunately, the gap between the largest eigenvalues of covtype-mult is very small $O(1E-5)$, making it a difficult eigenvalue problem. We compare PRIMME\_SVDS with Matlab SVDS function, a widely used SVD solver routine in the research community. We also set stopping tolerance \texttt{1E-5} to yield faster convergence for both solvers. We vary $R$ for SC\_RB from 16 to 128 and record the Accuracy (Acc) and Runtime (in Seconds) as our performance metrics for this set of experiment. 

\textbf{Results.} 
Fig. \ref{fig:Accu_runtime_varyingR_covtype-mult} shows how the accuracy (Acc) and runtime changes for SC\_RB using these two different SVD solvers when varying the rank $R$ on covtype-mult dataset. Interestingly, SC\_RB with PRIMME\_SVDS delivers more consistent accuracy than that with SVDS. This may be because Matlab SVDS function has a difficult time to converge to multiple very close singular values, showing an warning message "reach default maximum iterations". However, as shown in Fig. \ref{fig:Runtime_varyingR_covtype-mult}, less accurate singular triplets (from Matlab SVDS) takes significantly more computational time compared to PRIMME\_SVDS, especially when $R$ increases. In contrast, the computational time of eigendecomposition using PRIMME\_SVDS changes slowly with increased $R$. Thanks to the power of PRIMME, the proposed SC\_RB could achieve good clustering performance based on high-quality singular vectors while managing attractable computational time for very large sparse matrices.

\subsection{Scalability of SC\_RB when varying the number of data samples $N$}

\textbf{Setup.} Our goal in this experiment is to assess the scalability of SC\_RB when varying the number of data samples $N$ on poker dataset and another large dataset SUSY \footnote{SUSY is a large dataset in the LIBSVM data collections \cite{chang2011libsvm}.}. 
We vary the number of data samples in the range of $N = [100 \ \ 1,000,000]$ on poker and $N = [4000 \ \ 4,000,000]$ on the synthetic dataset.
We use the same hyperparameters as the previous experiments and fix $R = 256$. Since RB can be easily parallelized, we accelerate its computation using 4 threads. Matlab also automatically parallelize the matrix-vector operations for other solvers. We report the runtime for generating random binning feature matrix, computing partial eigendecomposition using state-of-the-art eigensolver \cite{stathopoulos2010primme,wu2017primme_svds}, performing K-means, and the overall runtime, respectively.

\textbf{Results.} Figure \ref{fig:scalability_varyingN} clearly shows that SC\_RB indeed scales linearly with the increase in the number of data samples.
Note that even for large datasets consisting of millions of samples, the computation time of SC\_RB is still less than 500 seconds. These results suggest that: (i) SC\_RB exhibits linear scalability in $N$ and is comtenant of handling large datasets in a reasonable time frame; (ii) with the state-of-the-art eigensolver \cite{wu2017primme_svds}, the complexity of computing a few of eigenvectors for spectral clustering is indeed linearly proportional to the matrix size $N$. 

The key factor that contributes to the competitive computation time and linear scalability of  SC\_RB in the data size $N$ is that we take into account the end-to-end spectral clustering pipeline consisting of the build blocks, RB generation, eigensolver and K-means, and our approach ensures each component takes a similar (linear) computation complexity, as analyzed in Section \ref{sec: scalable sc_rb}.

\subsection{Scalability of SC\_RB when varying the number of RB features $R$}

\textbf{Setup.} We further investigate the scalablity of various random-feature-based, sampling-based SC methods and approximate Kernel K-means methods when varying the latent feature size $R$. One of our goal is to investigate whether the latent matrix rank $R$ is linearly proportional to $N$. If this is true, then the total complexity of an approximation method is still bounded by $O(N^2)$, which is an unfavorable property for large-scale data. Therefore, we study the computational time of 8 baselines when varying $R$ from 16 to 1024 on 4 datasets. The other settings are same as before. 

\textbf{Results.} Figure \ref{fig:scalability_varyingR} shows that the computation of SC\_RB is as efficient as other approximation methods. There are several comments worth making here. First, compared to the quadratic complexity of exact SC in Figures \ref{fig:scalability_varyingR_pendigits} and \ref{fig:scalability_varyingR_letter}, various approximation methods except KK\_RF require much less computational time. It means that the total complexity of various methods are respecting to $O(NR)$ where $R$ could be somehow treated as a constant as long as $R$ is significantly smaller than $N$. Remarkably, Figures \ref{fig:scalability_varyingR_acoustic} and \ref{fig:scalability_varyingR_ijcnn1} show that most of approximation methods including SC\_RB can even behave as efficient as K-means on the original dataset. Obviously, most of these methods exhibit clear linear relation with $R$, indicating that these methods are not tightly associated with $N$. In other words, if the low rank $R$ in any method is respecting with $N$, then the total complexity of the method is proportional to $O(N^2)$, which should yield non-linear scalability respecting to $R$. The only exception is the KK\_RF method, which consistently requires much more runtime compared to other methods, making it less attractable than other methods.

\section{Conclusion}
In this paper, we have presented a scalable end-to-end spectral clustering method based on RB features (SC\_RB) for overcoming two computational bottlenecks - similarity graph construction and eigendecomposition of the graph Laplacian. By leveraging RB features, the pairwise similarity matrix can be approximated implicitly by the inner product of the RB feature matrix, which significantly reduces the computational complexity from quadratic to linear in terms of the number of data samples. We further show how to effectively and directly apply SVD  on the weighted RB feature matrix and introduce a state-of-the-art sparse SVD solver to efficiently manage the SVD computation for a very large sparse matrix. Our theoretical analysis shows that by drawing $R$ grids with at least $\kappa$ number of non-empty bins per grid, SC\_RB can guarantee convergence to exact spectral clustering with a rate of $O(1/(\kappa R))$ under the same pairwise graph construction process, which is much faster than other Random Features based SC methods.
Our extensive experiments on 8 benchmarks over 4 performance metrics demonstrate that SC\_RB either outperforms or matches 8 baselines in both accuracy and computational time, and corroborate that SC\_RB indeed exhibits linear scalability in terms of the number of data samples and the number of RB features.



\bibliographystyle{ACM-Reference-Format}
\bibliography{myrefs,RWS,SC_RB}
\end{document}